\documentclass[sigconf]{aamas} 

\usepackage{balance} 

\usepackage{amsmath,amssymb,amsthm,mathtools}
\usepackage{bm}  
\usepackage{booktabs,tabularx,longtable,multirow}
\usepackage{algorithm,algpseudocode}
\usepackage{graphicx}
\usepackage{xcolor}
\usepackage{siunitx}
\usepackage{caption,subcaption}
\usepackage{enumitem}
\usepackage{microtype}
\newcommand{\Proj}{\mathrm{Proj}}
\usepackage{balance}
\usepackage{hyperref}
\usepackage{cleveref}

\usepackage{mathtools}
\mathtoolsset{showonlyrefs}
\newtheorem{theorem}{Theorem}[section]
\newtheorem{lemma}[theorem]{Lemma}

\DeclareMathOperator*{\esssup}{ess\,sup}

\newcommand{\E}{\mathbb{E}}

\newcommand{\KL}{D_{\mathrm{KL}}}

\newcommand{\TV}{D_{\mathrm{TV}}}
\newcommand{\Ahat}{\widehat{A}}

\doi{6701030F}

\makeatletter
\gdef\@copyrightpermission{
  \begin{minipage}{0.2\columnwidth}
   \href{https://creativecommons.org/licenses/by/4.0/}{\includegraphics[width=0.90\textwidth]{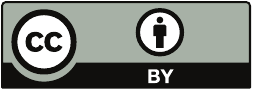}}
  \end{minipage}\hfill
  \begin{minipage}{0.8\columnwidth}
   \href{https://creativecommons.org/licenses/by/4.0/}{This work is licensed under a Creative Commons Attribution International 4.0 License.}
  \end{minipage}
  \vspace{5pt}
}
\makeatother

\setcopyright{ifaamas}
\acmConference[AAMAS '26]{Proc.\@ of the 25th International Conference
on Autonomous Agents and Multiagent Systems (AAMAS 2026)}{May 25 -- 29, 2026}
{Paphos, Cyprus}{Author list: Yi Xie, Yangyang Xu, Yi Fan(the work is not related to their position at Amazon Inc), Bo Liu}
\copyrightyear{2026}
\acmYear{2026}
\acmDOI{}
\acmPrice{}
\acmISBN{}

\title[AAMAS-2026 Formatting Instructions]{SAT: Sequential Agent Tuning for Coordinator‑Free Plug‑and‑Play Multi‑LLM Training with Monotonic Improvement Guarantees}

\author{Yi Xie}
\affiliation{
  \institution{Department of Electrical \& Computer Engineering, University of Arizona}
  \city{Tucson, Arizona}
  \country{USA}}
\email{yix@arizona.edu}

\author{Yangyang Xu}
\affiliation{
  \institution{Department of Mathematical Sciences, 
Rensselaer Polytechnic Institute}
  \city{Troy, New York}
  \country{USA}}
\email{xuy21@rpi.edu}

\author{Yi Fan}
\affiliation{
  \institution{Amazon Web Services}
  \city{New York}
  \country{USA}}
\email{fnyi@amazon.com}

\author{Bo Liu}
\affiliation{
  \institution{Department of Electrical \& Computer Engineering, University of Arizona}
  \city{Tucson, Arizona}
  \country{USA}}
\email{boliu@arizona.edu}

\begin{abstract}
Large language models (LLMs) with a large number of parameters achieve strong performance but are often prohibitively expensive to deploy. Recent work explores using teams of smaller, more efficient LLMs that collectively match or even outperform a single large model. However, jointly updating multiple agents introduces compounding distribution shifts, making coordination and stability during training difficult. We address this by introducing Sequential Agent Tuning (SAT), a coordinator-free training paradigm. SAT represents the team as a factorized policy and employs block-coordinate updates over agents, enabling scalable, decentralized training without a central controller. Specifically, we develop a sequence-aware, on-policy advantage estimator that conditions on the evolving team policy, coupled with per-agent KL trust regions that isolate occupancy drift. Theoretically, this framework provides two critical guarantees. First, it ensures monotonic improvement, stabilizing the training process. Second, it establishes provable plug-and-play invariance: any agent can be upgraded to a stronger model without retraining the rest of the team, with a formal guarantee that the performance bound improves. Empirically, a team of three 4B agents (12B total) trained with SAT surpasses the much larger Qwen3-32B on AIME24/25 benchmarks by 3.9\% on average. We validate our plug-and-play theory by swapping in two 8B agents, which boosts the composite score by 10.4\%. We provide code and appendix of proof at https://github.com/Yydc/SAT-AAMAS
\end{abstract}


\newcommand{\BibTeX}{\rm B\kern-.05em{\sc i\kern-.025em b}\kern-.08em\TeX}

\begin{document}


\pagestyle{fancy}
\fancyhead{}


\maketitle 


\section{Introduction}
Large language models (LLMs) have become effective problem solvers across a wide range of domains \citep{openai2023gpt4,wei2022chain,wang2023self,yao2023treeofthoughts,shinn2023reflexion}. Despite their impressive capabilities, state-of-the-art LLMs demand significant memory, computational, and energy resources, making them unavailable for resource-limited scenarios \citep{patterson2021carbon,alizadeh2024flash}. This disconnect between potential and deployability motivates a key question: \emph{Can teams of small, efficient models collectively achieve or surpass the performance of one well-tuned large model?}

Recent research has investigated multi-agent LLM systems by employing various coordination strategies, such as role assignment (planner, solver), hierarchical workflow, and ensemble refinement protocols \citep{wu2023autogen,li2023camel,yao2023treeofthoughts,wang2024moa,liu2023improving,liu2025grasp}. While these methods show promising empirical results, they are constrained by two main limitations. First, rely on predefined role assignment, which enforces interaction dynamics and may constrain the potential of LLMs systems \citep{yi2025from}. Second, existing methods generally lack robust theoretical foundations that guarantee multi-agent LLM systems achieve a strong ability, leaving open questions around why multi-LLMs \citep{wang2024moa, shinn2023reflexion}.

We address these gaps through a theoretical analysis of sequential agent training with factorized product policies and per-agent trust regions. We establish three properties under a sequence-aware, on-policy advantage estimator that conditions on intermediate policies. Sequence-agnostic improvement states that stage-wise bounds hold regardless of the agent update sequence chosen before each update within the stage. Plug-and-play invariance ensures that agents can be upgraded via a stage-0 KL projection without retraining others while preserving certificates. Certificate tightening indicates that upgrades can increase surrogate values at fixed radii or achieve comparable gains with smaller radii, thereby reducing the cumulative penalty that scales with the sum of the square roots of the radii. The main challenge, covariate shift arising from sequential updates, is addressed by evaluating the advantages under the current intermediate occupancy and constraining the per-agent per-state KL divergence; we also demonstrate that naive rollout reuse without such conditioning fails to guarantee improvement, whereas our approach yields monotonic bounds.

Empirical evaluation across general reasoning, active reasoning and planning (AutoLogi PlanBench ) demonstrates that SAT-trained small-model teams match or exceed large baselines under matched evaluation protocols. These gains align with our order‑agnostic stage bounds and certificate tightening under per‑agent KL trust regions. We summarize our main contributions as follows:
\begin{itemize}
\item We present a theoretical framework for multi-agent LLM finetuning with monotonic improvement bounds, sequence-agnostic guarantees, and plug-and-play invariance (Sec.\ref{sec:theory}). 
\item We instantiated the theoretical framework through Sequential Agent Tuning (SAT), a coordinator-free training paradigm utilizing sequence-aware block optimization (Sec.\ref{sec:algorithm}).
 \item We demonstrate that SAT-trained teams outperform strong baselines on 7 benchmarks, with further improvements achieved through plug-and-play paradigm upgrades (Sec .~\ref {sec:experiment}).
\end{itemize}
\section{Related Work}

\paragraph{Multi-agent LLM Systems and Orchestration.}
A line of work explores orchestrating multiple LLMs via roles, tool use, and multi-turn protocols, often with an explicit judge or controller that assigns subtasks or aggregates opinions \citep{yao2022react}.
Ensembling-style procedures improve reliability by aggregating diverse completions before selection \citep{wang2023self}.
Structured search over thoughts extends this idea with deliberate branching and pruning at the sequence level \citep{yao2023treeofthoughts}.
Planning-inspired search further guides generation with MCTS-style rollouts that critique and refine candidates \citep{feng2023alphazero, xie2025haer}.
In contrast, our method is coordinator-free and employs stage-wise sequential updates with per-agent trust regions and an analysis-driven surrogate, yielding provable joint-stage improvement.

\paragraph{Monotonic Policy Improvement and Trust Regions.}
Conservative Policy Iteration establishes monotonic improvement under carefully controlled policy updates \citep{kakade2002cpi}.
Trust Region Policy Optimization enforces KL trust regions to stabilize updates with theoretical guarantees \citep{schulman2015trpo}.
Proximal Policy Optimization adopts a clipped surrogate that approximates trust-region behavior in practice \citep{schulman2017proximal}. And some early works applying BCD to RL includes~\citep{zhang2021mean,xie2018block}. 
Our development departs by updating agents sequentially while conditioning analysis on the intermediate product policy, and by providing a joint-stage lower bound where per-agent KL penalties accumulate as \(\sum_{i=1}^{n}\sqrt{\delta_i}\), justifying small per-agent trust regions.

\paragraph{Off-Policy Advantage Estimation under Distribution Shift.}
Generalized Advantage Estimation (GAE) trades bias for variance reduction by utilizing multi-step returns and a value baseline \citep{schulman2016gae, liu2023improving}.
Retrace introduces safe multi-step targets via truncated importance weights for off-policy data \citep{munos2016safe, xie2025acorn}.
V-trace extends this idea to scalable actor–critic training with robust corrections \citep{espeholt2018impala}.
In LLM pipelines, iterative procedures such as self-consistency can shift the data distribution across rounds \citep{wang2023self}.
Sequence-level search similarly modifies the sampling distribution as the tree expands and prunes candidates \citep{yao2023treeofthoughts}.
Our estimator conditions on the updated intermediate policy and uses clipped multi-step ratios to manage bias, leveraging sequence information gathered within a stage.

\paragraph{Dynamic Routing and Plug-and-Play Compute.}
Cascaded inference routes easier inputs to cheaper models and escalates only when necessary, reducing latency without sacrificing accuracy \citep{leviathan2023fast}.
Previous multi-agent frameworks may swap tools or models, but they rarely provide guarantees that persist after replacement. Our plug-and-play analysis shows that monotonic improvement certificates remain invariant under agent replacement, provided that the surrogate and trust-region constraints are respected.

\section{Preliminaries}
\label{sec:prelim}

\paragraph{Environment and policies.}
Let $\mathcal{M}=(\mathcal{S},\{\mathcal{A}_i\}_{i=1}^{n},P,r,\gamma)$ be a discounted MDP with $\gamma\in(0,1)$ and bounded rewards $|r|\le R_{\max}$. The joint action space is the product $\mathcal{A}=\mathcal{A}_1\times\cdots\times\mathcal{A}_n$. A team of $n$ execution agents forms a factorized policy over \emph{joint actions}
\[
\pi(a\mid s)=\prod_{i=1}^{n}\pi^{(i)}(a_i\mid s),\qquad a=(a_1,\ldots,a_n)\in\mathcal{A}.
\]
Thus $\pi(\cdot\mid s)$ is a valid distribution on $\mathcal{A}$, not $n$ agents “taking the same action.” Let $d^{\pi}$ denote the discounted state visitation measure:
\[
J(\pi)=\E_{\mu,s_0\sim\mu,\,\pi}\Big[\sum_{t\ge 0}\gamma^{t}r_t\Big]
=\frac{1}{1-\gamma}\;\E_{s\sim d^{\pi},\,a\sim\pi(\cdot\mid s)}[\,r(s,a)\,].
\]
In practice, some tasks activate only a subset of agents at a given state (for example, token heads, tool routers, or role-switching schedulers). To faithfully cover both simultaneous and interleaved execution while preserving the product-policy analysis, we adopt a masked activation view: at each state $s$, let $\mathcal{I}(s)\subseteq\{1,\ldots,n\}$ denote the set of active agents and interpret the joint policy over active heads as $\pi(a\mid s)=\prod_{i\in\mathcal{I}(s)}\pi^{(i)}(a_i\mid s)$, while inactive heads take a fixed no-op. All per-state divergences and trust-region constraints below are evaluated over the active heads at that state; when all agents act simultaneously, $\mathcal{I}(s)=\{1,\ldots,n\}$ and the definitions reduce to the standard full-product case. This formalization of masking is consistent with our later sequence-level objective and does not alter the statements that only depend on adjacent intermediate policies. If an implementation never uses masked activation, setting $\mathcal{I}(s)$ to the full set recovers the original expressions.

We write $Q^{\pi},V^{\pi}$ for action/value functions and $A^{\pi}(s,a)=Q^{\pi}(s,a)-V^{\pi}(s)$. We update agents sequentially in order $\sigma(1),\ldots,\sigma(n)$:
\[
\hat\pi^{0}=\pi_{\mathrm{cur}},\;
\hat\pi^{i}=(\pi^{\sigma(1)}_{\mathrm{tar}},\ldots,\pi^{\sigma(i)}_{\mathrm{tar}},\pi^{\sigma(i+1)}_{\mathrm{cur}},\ldots,\pi^{\sigma(n)}_{\mathrm{cur}}),\;
\bar\pi=\hat\pi^{n}.
\]

\paragraph{Per-state divergences and trust regions.}
For policies $\pi_1,\pi_2$ define the per-state maximal divergences
\begin{align*}
\KL^{\max}(\pi_1\Vert\pi_2) &= \sup_{s}\KL\!\big(\pi_1(\cdot\mid s)\Vert\pi_2(\cdot\mid s)\big), \\
\TV^{\max}(\pi_1\Vert\pi_2) &= \sup_{s}\tfrac12\|\pi_1(\cdot\mid s)-\pi_2(\cdot\mid s)\|_1.
\end{align*}
and use Pinsker's inequality to link them: $\TV^{\max}\le \sqrt{\tfrac12\,\KL^{\max}}$. At step $i$, agent $\sigma(i)$ obeys the per-state trust region
\[
\KL^{\max}\!\big(\pi^{\sigma(i)}_{\mathrm{tar}}\Vert\pi^{\sigma(i)}_{\mathrm{cur}}\big)\le \delta_i
\quad(\text{uniform case: }\delta_i\equiv\delta).
\]
We also define $A_{\max}\coloneqq \sup_{s,a}|A^{\pi}(s,a)|\le \tfrac{2R_{\max}}{1-\gamma}$ for later bounds.

\paragraph{Sequence-aware surrogate objective.}
When updating agent $\sigma(i)$ we evaluate on-policy under the current intermediate occupancy:
\[
L_i^{\textsc{seq}}(\pi^{\sigma(i)}_{\mathrm{tar}})
=\frac{1}{1-\gamma}\;
\E_{s\sim d^{\hat\pi^{\,i-1}},\,a\sim \hat\pi^{\,i}(\cdot\mid s)}
\big[\Ahat^{\,i-1}_{\textsc{on}}(s,a)\big],
\]
where $\Ahat^{\,i-1}_{\textsc{on}}$ is a multi-step on-policy estimator computed from trajectories of $\hat\pi^{\,i-1}$ (for example, GAE with trace $\lambda$). This evaluates full episodes under the intermediate policy and is consistent with masked activation, as interpreted by considering the joint action as the product of the active heads at each visited state.

\paragraph{Two standard tools.}
The performance difference lemma states that $J(\pi')-J(\pi)=\tfrac{1}{1-\gamma}\,\E_{s\sim d^{\pi'},\,a\sim\pi'}[A^{\pi}(s,a)]$. An occupancy-shift bound for bounded test functions $f$ is given by $\big|\E_{d^{\pi'}}[f]-\E_{d^{\pi}}[f]\big|\le \tfrac{2\gamma}{1-\gamma}\,\TV^{\max}(\pi'\Vert\pi)\,\|f\|_{\infty}$.

\section{Theoretical Framework}
\label{sec:theory}

We study a coordinator-free, plug-and-play training paradigm in which a team of $n$ execution agents is updated sequentially via block-coordinate ascent. The main technical challenge is distribution shift: when one agent updates, the effective evaluation distribution for subsequent agents changes.
We address this issue with a sequence-aware, on-policy advantage estimator that conditions on the current intermediate team policy, along with per-agent. These per-state KL trust regions limit occupancy drift. These ingredients yield single-step and joint-stage monotonic improvement bounds, sequence-agnostic guarantees compatible with learned schedulers and plug-ins, and information-theoretic envelopes that quantify what is achievable under a fixed sampling budget. Full proofs are presented in the Appendix, with a table that maps the results.

\subsection{Main Theoretical Guarantee}
\label{subsec:main-guarantee}

\noindent\textbf{Statement (one stage, sequence-agnostic, plug-and-play).}
For a full training stage of $n$ sequential agent updates with per-agent, per-state KL radii $\{\delta_i\}_{i=1}^n$ and on-policy budgets $\{N_i\}_{i=1}^n$, we provide a high-probability certificate of monotonic improvement. In particular, with probability at least $1-\delta_{\mathrm{conf}}$,
%
\[
\begin{aligned}
J(\bar\pi)-J(\pi_{\mathrm{cur}}) \ \ge\ 
&\ \underbrace{\sum_{i=1}^{n}\Big(\kappa_i^{\mathrm{reg}}\sqrt{\delta_i}-a_i^{\mathrm{reg}}\delta_i\Big)}_{\text{Information-Geometric Gain}} 
- \underbrace{\frac{2\gamma}{(1-\gamma)^2}A_{\max}\sum_{i=1}^{n}\sqrt{\tfrac12\,\delta_i}}_{\text{Occupancy-Shift Penalty}} \\
& - \underbrace{\frac{1}{1-\gamma}\sum_{i=1}^{n}\zeta_i}_{\text{Estimator-Bias Penalty}} 
- \underbrace{\sum_{i=1}^{n}\frac{A_{\max}}{1-\gamma}\sqrt{\frac{\log(2n/\delta_{\mathrm{conf}})}{2N_i}}}_{\text{Finite-Sample Error}}.
\end{aligned}
\]

where the gain coefficients $\kappa_i^{\mathrm{reg}}=\sqrt{2\,g_i^{\!\top}(F_i^{\mathrm{reg}})^{-1}g_i}$ and curvature terms $a_i^{\mathrm{reg}}=L_i^{\mathrm{loc}}/\lambda_{\min}(F_i^{\mathrm{reg}})$ are defined in Theorem~\ref{thm:info-lower}. The three penalty terms have precise interpretations: the \emph{occupancy-shift penalty} $\frac{2\gamma}{(1-\gamma)^2}A_{\max}\sum_{i=1}^{n}\sqrt{\tfrac12\,\delta_i}$ captures the cumulative cost of distribution shift as agents update sequentially (Theorem~\ref{thm:joint}); the \emph{estimator-bias penalty} $\frac{1}{1-\gamma}\sum_{i=1}^{n}\zeta_i$ arises from our sequence-aware advantage estimator (Theorem~\ref{thm:single}); and the \emph{finite-sample error} $\sum_{i=1}^{n}\frac{A_{\max}}{1-\gamma}\sqrt{\frac{\log(2n/\delta_{\mathrm{conf}})}{2N_i}}$ accounts for statistical uncertainty from using $N_i$ samples per step (Theorem~\ref{thm:finite-sample}).

Crucially, this policy improvement guarantee is realized via a provably convergent optimization procedure. The sequential projected block updates on the stage surrogate $G(\theta)=\sum_i L_i^{\textsc{seq}}$ satisfy the standard $O(1/K)$ convergence rate for the projected-gradient mapping over $K$ block steps (Theorem~\ref{thm:convergence}). All statements are \emph{sequence-agnostic} in $\sigma$ and remain valid under \emph{plug-and-play} replacements, provided the new agent is initialized within the same per-state KL trust region (via Stage-0 alignment in the Appendix).

\paragraph{Remark (High-level interpretation).}
The inequality above certifies that the \emph{information-geometric gain}—scaled by natural-gradient structure through $\sqrt{\delta_i}$ terms—minus three \emph{controllable costs} (distribution shift, estimator bias, and sampling noise) remains nonnegative in aggregate. At the same time, the optimizer reliably finds such updates at a rate $O(1/K)$. Proof ingredients are modular: occupancy shift via per-state TV/KL bounds (§\ref{subsec:occ-shift-main}), single-step $\Rightarrow$ joint-stage telescoping (Theorems~\ref{thm:single}--\ref{thm:joint}), information-geometric lower bounds for $\kappa_i^{\mathrm{reg}}, a_i^{\mathrm{reg}}$ (Theorem~\ref{thm:info-lower}), finite-sample concentration (Theorem~\ref{thm:finite-sample}), and sequential projected-gradient convergence (Theorem~\ref{thm:convergence}). For tighter certificates, one may replace $\delta_i$ by any expected-KL radius $\bar\delta_i\le \delta_i$ in the gain terms (Theorem~\ref{thm:info-lower}).

\subsection{Occupancy Shift: definition and per-state control}
\label{subsec:occ-shift-main}

\paragraph{Definition (discounted occupancy shift).}
For two policies $\pi'$ and $\pi$, the discounted occupancy shift measures how much their state visitation distributions differ:
\[
\Delta_{\mathrm{occ}}(\pi',\pi;f)\ \triangleq\ \big|\E_{d^{\pi'}}[f]-\E_{d^{\pi}}[f]\big|,
\qquad
\Delta_{\mathrm{occ}}(\pi',\pi)\ \triangleq\ \sup_{\|f\|_{\infty}\le 1}\ \Delta_{\mathrm{occ}}(\pi',\pi;f).
\]

\begin{lemma}[Occupancy-shift bound]\label{lem:occ-shift}
With per-state divergences
$\TV^{\max}(\pi'\Vert\pi)\!\triangleq\!\esssup_{s}\,\TV\!\big(\pi'(\cdot|s),\pi(\cdot|s)\big)$
and
$\KL^{\max}(\pi'\Vert\pi)\!\triangleq\!\esssup_{s}\,\KL\!\big(\pi'(\cdot|s)\Vert\pi(\cdot|s)\big)$,
one has the worst-case bound
\[
\Delta_{\mathrm{occ}}(\pi',\pi)\ \le\ \frac{2\gamma}{1-\gamma}\,\TV^{\max}(\pi'\Vert\pi)
\ \le\ \frac{2\gamma}{1-\gamma}\,\sqrt{\tfrac12\,\KL^{\max}(\pi'\Vert\pi)}.
\]
\end{lemma}

When plugged into the performance-difference identity (which carries an extra factor $1/(1-\gamma)$), this yields a \emph{single-step penalty}
$\tfrac{2\gamma}{(1-\gamma)^2}\,A_{\max}\sqrt{\tfrac12\,\KL^{\max}}$
in Theorem~\ref{thm:single}.
\paragraph{Per-state control (our remedy).}
At sequential step $i$, we \emph{cap} the per-state KL between the updated agent and its current version:
\[
\KL\!\big(\pi^{\sigma(i)}_{\mathrm{tar}}(\cdot|s)\,\big\|\,\pi^{\sigma(i)}_{\mathrm{cur}}(\cdot|s)\big)\ \le\ \delta_i(s)\quad\forall s.
\]
This implies $\KL^{\max}\!\big(\hat\pi^{\,i}\Vert\hat\pi^{\,i-1}\big)\le \delta_i^{\max}:=\esssup_s\delta_i(s)$ and therefore
\[
\Delta_{\mathrm{occ}}\!\big(\hat\pi^{\,i},\hat\pi^{\,i-1}\big)\ \le\ \frac{2\gamma}{1-\gamma}\,\sqrt{\tfrac12\,\delta_i^{\max}}.
\]
Consequently, the only distribution-shift cost entering our single-step and joint-stage guarantees is the explicit $\sqrt{\delta_i}$ penalty (Theorems~\ref{thm:single} and \ref{thm:joint}), which we control directly by choosing the per-state radii $\{\delta_i(s)\}$ (and enforcing them via trust-region updates / Stage-0 KL projection).
\subsection{Monotonic Improvement Guarantees}
\begin{theorem}[Single-step monotonic improvement]\label{thm:single}
For step $i\in\{1,\ldots,n\}$,
\[
J(\hat\pi^{\,i})-J(\hat\pi^{\,i-1})
\ \ge\
L_i^{\textsc{seq}}(\pi^{\sigma(i)}_{\mathrm{tar}})
-\frac{2\gamma}{(1-\gamma)^2}\,A_{\max}\sqrt{\tfrac12\,\KL^{\max}\!\big(\hat\pi^{\,i}\Vert\hat\pi^{\,i-1}\big)}
-\frac{\zeta_i}{1-\gamma}.
\]
Since $A_{\max}\le \tfrac{2R_{\max}}{1-\gamma}$, the penalty equals
$\tfrac{4\gamma R_{\max}}{(1-\gamma)^{3}}\sqrt{\tfrac12\,\KL^{\max}\!\big(\hat\pi^{\,i}\Vert\hat\pi^{\,i-1}\big)}$.
\end{theorem}

\begin{proof}[Proof sketch]
The performance difference lemma gives
\[
J(\hat\pi^{\,i})-J(\hat\pi^{\,i-1})
=\frac{1}{1-\gamma}\E_{s \sim d^{\hat\pi^{\,i}}, a\sim \hat\pi^{\,i}(\cdot|s)}[A^{\hat\pi^{\,i-1}}(s,a)].
\]
Apply the Lemma ~\ref{lem:occ-shift} above to replace $d^{\hat\pi^{\,i}}$ by $d^{\hat\pi^{\,i-1}}$ (TV/KL penalty). Then replace $A^{\hat\pi^{\,i-1}}$ by the on-policy estimator and account for $\zeta_i$. Full details are in the Appendix (Single-step).
\end{proof}

\begin{lemma}[KL accumulation for product policies]\label{lem:klstep}
Let $p(\mathbf a|s)=\prod_{j=1}^n p_j(a^{(j)}|s)$ and $q(\mathbf a|s)=\prod_{j=1}^n q_j(a^{(j)}|s)$. Then
\[
  \KL\!\big(p(\cdot|s)\Vert q(\cdot|s)\big)=\sum_{j=1}^n \KL\!\big(p_j(\cdot|s)\Vert q_j(\cdot|s)\big).
\]
At step $i$, $\hat\pi^{\,i}$ and $\hat\pi^{\,i-1}$ differ only in factor $\sigma(i)$, hence for all $s$,
\[
  \KL\!\big(\hat\pi^{\,i}(\cdot|s)\Vert \hat\pi^{\,i-1}(\cdot|s)\big)
  =\KL\!\big(\pi^{\sigma(i)}_{\mathrm{tar}}(\cdot|s)\Vert\pi^{\sigma(i)}_{\mathrm{cur}}(\cdot|s)\big)\le \delta_i.
\]
\end{lemma}

\begin{theorem}[Joint-stage monotonic improvement]\label{thm:joint}
After a stage with radii $\{\delta_i\}_{i=1}^{n}$,
\[
J(\bar\pi)-J(\pi_{\mathrm{cur}})
\ \ge\
\sum_i L_i^{\textsc{seq}}
-\frac{2\gamma}{(1-\gamma)^2}\sum_i A_{\max}^{(i)}\sqrt{\tfrac12\,\delta_i}
-\frac{1}{1-\gamma}\sum_i \zeta_i.
\]
In the uniform case $\delta_i\equiv\delta$, the penalty scales as $O(n\sqrt{\delta})$, where $A_{\max}^{(i)} \triangleq \sup_{s,a}\big|A^{\hat\pi^{\,i-1}}(s,a)\big|$.
\end{theorem}

\begin{proof}[Proof sketch]
Telescoping gives $J(\bar\pi)-J(\pi_{\mathrm{cur}})=\sum_{i=1}^{n}[J(\hat\pi^{\,i})-J(\hat\pi^{\,i-1})]$. Apply Theorem~\ref{thm:single} to each term. KL penalties add directly because each step constrains only one agent (Lemma~\ref{lem:klstep}); there are no cross-terms. Full details are in the Appendix (Joint-stage).
\end{proof}

\paragraph{Structural properties.}
Sequence-agnosticism: Theorem~\ref{thm:joint} holds for any update order $\sigma$, including data-dependent choices; the numeric value of the lower bound can still depend on the realized order.
Plug-and-play invariance: Replacing agents with stronger models preserves guarantees if they optimize the same surrogate under the same constraints; this is realized via a Stage‑0 alignment that starts within the trust region (Appendix Stage‑0).
Certificate tightening: Upgrades either increase $\sup L_i^{\textsc{seq}}$ at fixed $\delta_i$ or achieve the same surrogate with smaller radii $\delta'_i<\delta_i$, reducing $\sum_i\sqrt{\delta_i}$. A high-probability relaxation can replace $\delta_i$ with $\delta_i+\epsilon(N_i,\eta)$, where
$\epsilon(N_i,\eta)=\sqrt{\log(2/\eta)/(2N_i)}$ (DKW; see Appendix Stage‑0).

\subsection{Information-Theoretic Envelopes}
\label{subsec:info-envelopes}

We now establish fundamental limits on achievable improvements under KL constraints and finite sampling budgets. Let $N_i$ denote the number of on-policy episodes at step $i$ under $d^{\hat\pi^{\,i-1}}$.

\paragraph{Information-geometric preliminaries.}
Assume $L_i^{\textsc{seq}}$ is twice differentiable wrt the parameters $\theta_i$ of agent $\sigma(i)$ near $\pi^{\sigma(i)}_{\mathrm{cur}}$. Let $g_i = \nabla_{\theta_i} L_i^{\textsc{seq}}$ and
\[
F_i = \E_{s \sim d^{\hat\pi^{\,i-1}}}\left[\E_{a \sim \pi_{\theta_i}(\cdot|s)}\left[\nabla_{\theta_i}\log\pi_{\theta_i}(a|s)\nabla_{\theta_i}\log\pi_{\theta_i}(a|s)^\top\right]\right].
\]
Local smoothness: $L_i^{\textsc{seq}}(\theta_i+\Delta) \ge L_i^{\textsc{seq}}(\theta_i) + g_i^\top\Delta - \frac{L_i^{\mathrm{loc}}}{2}\|\Delta\|^2$.
Fisher–KL bridge: $\E_{s}[\KL(\pi_{\theta_i+\Delta}\|\pi_{\theta_i})]\approx \frac{1}{2}\Delta^\top F_i \Delta$ for small $\|\Delta\|$.
\begin{theorem}[Oracle single-step upper bound]
\label{thm:oracle-upper}
\mbox{}\\
Under $\KL^{\max}(\hat\pi^{\,i}\Vert\hat\pi^{\,i-1}) \le \delta_i$ and $|A|\le A_{\max}$,
\[
J(\hat\pi^{\,i})-J(\hat\pi^{\,i-1})
\ \le\ \frac{A_{\max}}{1-\gamma}\,\sqrt{2\,\delta_i}.
\]
\end{theorem}

\begin{theorem}[Finite-budget single-step envelope]
\label{thm:budget-upper}
With $N_i$ on-policy episodes at step $i$, with probability at least $1-\delta$,
\[
J(\hat\pi^{\,i})-J(\hat\pi^{\,i-1})
\ \le\ \frac{A_{\max}}{1-\gamma}\sqrt{2\,\delta_i}
\ +\ \frac{A_{\max}}{1-\gamma}\sqrt{\frac{\log(2/\delta)}{2N_i}}.
\]
\end{theorem}

\begin{theorem}[Budget-aware stage lower bound]
\label{thm:info-lower}
Under the local smoothness and Fisher–KL bridge assumptions, with $F_i^{\mathrm{reg}} = F_i + \epsilon I$ and noting that $\E_s\KL\le \KL^{\max}$ (so any effective radius $\bar\delta_i\le\delta_i$ may be used), with probability at least $1-\delta_{\mathrm{conf}}$,
\[
\begin{aligned}
J(\bar\pi)-J(\pi_{\mathrm{cur}})
\ \ge\ 
&\sum_{i=1}^{n}\big(\kappa_i^{\mathrm{reg}}\sqrt{\bar\delta_i}-a_i^{\mathrm{reg}}\bar\delta_i\big)
-\ \frac{2\gamma}{(1-\gamma)^2}A_{\max}\sum_{i=1}^{n}\sqrt{\tfrac12\,\delta_i} \\
&\ -\ \frac{1}{1-\gamma}\sum_{i=1}^{n}\zeta_i
\ \ -\ \sum_{i=1}^{n}\frac{A_{\max}}{1-\gamma}\sqrt{\frac{\log(2n/\delta_{\mathrm{conf}})}{2N_i}},
\end{aligned}
\]
where $\kappa_i^{\mathrm{reg}} = \sqrt{2\,g_i^\top (F_i^{\mathrm{reg}})^{-1}g_i}$ and $a_i^{\mathrm{reg}} = L_i^{\mathrm{loc}}/\lambda_{\min}(F_i^{\mathrm{reg}})$.
\end{theorem}

\subsection{Convergence and Finite-Sample Analysis}

\paragraph{Stage objective and block updates.}
Let $\theta=(\theta_1,\ldots,\theta_n)$ parameterize the team policy $\pi_\theta=\prod_{j=1}^n\pi^{(j)}_{\theta_j}$.
Given the sequence-aware surrogates $\{L_i^{\textsc{seq}}\}_{i=1}^n$ under the intermediate occupancies $\{d^{\hat\pi^{\,i-1}}\}$, define
\[
G(\theta)\ \triangleq\ \sum_{i=1}^{n}L_i^{\textsc{seq}}(\theta)\ =\ \sum_{j=1}^{n}G_j(\theta_j),
\qquad
G_j(\theta_j):=L^{\textsc{seq}}_{\sigma^{-1}(j)}(\theta_j).
\]
By construction, each $L_i^{\textsc{seq}}$ depends only on block $\theta_{\sigma(i)}$ under fixed $d^{\hat\pi^{\,i-1}}$, hence mixed second derivatives vanish.

At the \emph{start of the stage}, fix per-agent trust regions.
\[
\Theta_j\ :=\ \Big\{\theta_j:\ \KL\!\Big(\pi^{(j)}_{\theta_j}(\cdot\mid s)\,\big\|\,\pi^{(j)}_{\mathrm{cur}}(\cdot\mid s)\Big)\ \le\ \delta_j(s)\ \ \forall s\Big\},
\qquad \Theta:=\prod_{j=1}^n\Theta_j,
\]
which remain \emph{fixed within the stage}. Perform one sweep of \emph{sequential} block updates following $\sigma$:
\[
\theta_{j}^{\,i}\ =\ \Proj_{\Theta_j}\!\Big(\theta_{j}^{\,i-1}+\eta_j\,\nabla_{\theta_j}G(\theta^{\,i-1})\Big),\qquad
\theta_{-j}^{\,i}=\theta_{-j}^{\,i-1},\quad j=\sigma(i),
\]
And define the block projected-gradient mapping.
$
g_{\eta_j}^{(j)}(\theta^{\,i-1})\ :=\ \tfrac{1}{\eta_j}\big(\theta_{j}^{\,i}-\theta_{j}^{\,i-1}\big).
$

\begin{theorem}[Sequential block-coordinate gradient descent (BCGD) on a fixed trust region: monotonicity and rate]
\label{thm:convergence}
Assume $|\widehat A^{\,i-1}_{\textsc{on}}|\le A_{\max}$ and bounded log-policy derivatives
$\|\nabla_{\theta}\log\pi_{\theta}(\cdot\mid s)\|\le B_1$,
$\|\nabla_{\theta}^2\log\pi_{\theta}(\cdot\mid s)\|_{\mathrm{op}}\le B_2$
on $\Theta$.
Then each $G_j$ is $L_{\mathrm{blk}}$-smooth with
$L_{\mathrm{blk}}=\frac{A_{\max}}{1-\gamma}(B_2+B_1^2)$.
If $\eta_j\le 1/L_{\mathrm{blk}}$, every block step is ascent and
\[
G(\theta^{\,i})-G(\theta^{\,i-1})
\ \ge\ \Big(\eta_j-\tfrac{L_{\mathrm{blk}}\eta_j^2}{2}\Big)\,\big\|g_{\eta_j}^{(j)}(\theta^{\,i-1})\big\|^2
\ \ge\ \tfrac{\eta_j}{2}\,\big\|g_{\eta_j}^{(j)}(\theta^{\,i-1})\big\|^2.
\]
After $K$ block updates (e.g., one sweep has $K=n$),
\[
\frac{1}{K}\sum_{t=1}^{K}\big\|g_{\eta_{j_t}}^{(j_t)}(\theta^{\,t-1})\big\|^2
\ \le\ \frac{2\,(G^\star-G(\theta^{\,0}))}{(\min_t \eta_{j_t})\,K},\qquad
G^\star=\sup_{\theta\in\Theta}G(\theta).
\]
\emph{Proof sketch.} Block smoothness from bounded Hessians, one-step descent lemma on the active block, and Euclidean projection optimality; telescope across blocks. Full proof: Appendix (PGD).
\end{theorem}

\begin{theorem}[Finite-sample concentration for $L_i^{\textsc{seq}}$]
\label{thm:finite-sample}
Let $\widehat{L}_i^{\textsc{seq}}$ be the empirical surrogate from $N_i$ on-policy episodes at step $i$ and set $B:=A_{\max}/(1-\gamma)$. For any $\delta\in(0,1)$, with i.i.d. episodes,
\[
\big|\widehat L_i^{\textsc{seq}}-L_i^{\textsc{seq}}\big|
\ \le\ B\sqrt{\frac{\log(2/\delta)}{2N_i}}
\quad\text{with probability at least }1-\delta.
\]
Under $\beta$-mixing episodes with $\sum_{t\ge1}\beta(t)<\infty$, the same bound holds with
$N_i$ replaced by $N_{i,\mathrm{eff}}=N_i/(1+2\sum_{t\ge1}\beta(t))$.
Variance-aware and robust MoM variants are given in the Appendix: Finite‑sample.
\end{theorem}

\section{Algorithm}
\label{sec:algorithm}

We instantiate our framework with SAT, which updates agents sequentially using sequence-level optimization. Unlike token-level baselines (e.g., PPO), SAT uses a group-relative, sequence-level objective that evaluates complete trajectories and normalizes advantages within prompt-level groups such as GRPO and DAPO \citep{shao2024deepseekmath,yu2025dapo}. This is consistent with the masked-activation product-policy view (Sec.\ref{sec:prelim}): only active heads contribute to joint actions and divergences at each state, and standard factorization is recovered in the fully simultaneous case. Each part below reflects an assumption or value in Sec .~\ref {sec:theory}: the sequence-level surrogate targets $L_i^{\textsc{seq}}$, clipped advantages ensure boundary conditions for $A_{\max}$, and the trust-region controller enforces per-agent per-state $\KL^{\max}$.

\paragraph{Sequence-level optimization.}
Token-level RLHF methods compute advantages per action and update policies using per-step likelihood ratios. In multi-agent settings with sequential updates, such token-wise advantages may induce off-policy errors as agents update one by one; credit assignment across agents is also ambiguous. SAT operates at the sequence level, evaluating complete trajectories and normalizing advantages within groups of completions for the same prompt. This matches our theory, where the surrogate $L_i^{\textsc{seq}}$ evaluates full episodes under the intermediate policy $\hat{\pi}^{\,i-1}$, reducing distribution mismatch and fitting stage-wise bounds in Sec .~\ref {sec:theory}.

\paragraph{Sequential updates with intermediate policies.}
At each stage, agents are updated in order $\sigma(1),\ldots,\sigma(n)$. When updating agent $\sigma(i)$, we form the intermediate policy
\begin{equation}
\hat{\pi}^{\,i-1}
=\big(\pi^{\sigma(1)}_{\mathrm{tar}},\ldots,\pi^{\sigma(i-1)}_{\mathrm{tar}},\;
\pi^{\sigma(i)}_{\mathrm{cur}},\ldots,\pi^{\sigma(n)}_{\mathrm{cur}}\big).
\end{equation}
The sequence-aware advantage estimator conditions on $\hat{\pi}^{\,i-1}$. At the beginning of each stage, we collect fresh on-policy rollouts under $\hat{\pi}^{\,0}=\pi_{\mathrm{cur}}$, and within the stage, we reuse these trajectories by recomputing per-timestep importance ratios
\[
\rho_t=\frac{\hat{\pi}^{\,i-1}(a_t\mid s_t)}{\pi_{\mathrm{cur}}(a_t\mid s_t)}
\quad\text{with truncated weights}\quad
c_t=\min\{1,\rho_t\}
\]
in the multi-step estimator; advantages use GAE with $\lambda=0.95$. This realizes the “sequence-aware” evaluation required by Sec .~\ref {sec:theory}, and the truncation-induced bias is explicitly captured by the estimator term $\zeta_i$ that appears in the improvement bounds.

\paragraph{Group-relative normalization.}
For each prompt, we sample $G_{\text{grp}}$ trajectories under $\hat{\pi}^{\,i-1}$ and compute the group-normalized advantage $\tilde A_{g}$ for $g\in\{1,\ldots,G_{\text{grp}}\}$:
\begin{equation}
    \tilde A_{g} = \frac{\widehat A^{\,i-1}_{g}-\mu}{\sigma+\epsilon}, \quad 
    \mu = \frac{1}{G_{\text{grp}}}\sum_{j=1}^{G_{\text{grp}}}\widehat A^{\,i-1}_{j}, \quad 
    \sigma^2 = \frac{1}{G_{\text{grp}}}\sum_{j=1}^{G_{\text{grp}}}(\widehat A^{\,i-1}_{j}-\mu)^2.
\end{equation}
Here, $\widehat A^{\,i-1}_{g}$ is the aggregated per-timestep advantage. We use $G_{\text{grp}}\in\{4,8\}$ to balance variance and cost. Finally, we apply symmetric clipping $|\tilde A_{g}|\le A_{\text{clip}}$ (setting $A_{\max}:=A_{\text{clip}}$) to satisfy the bounded-advantage assumption in Sec.~\ref{sec:theory}.

\paragraph{Trust-region update for agent $\sigma(i)$.}
Let $\mathcal{T}_i(\tau)$ denote the timesteps where agent $\sigma(i)$ acts along trajectory $\tau$. 
\begin{align}
u_i(\tau) &= \sum_{t\in\mathcal{T}_i(\tau)}\Big(\log\pi^{\sigma(i)}(a_{i,t}\!\mid s_t)-\log\pi^{\sigma(i)}_{\mathrm{cur}}(a_{i,t}\!\mid s_t)\Big), \\
r_i(\tau) &= \exp\big(u_i(\tau)\big).
\end{align}
We optimize a clipped sequence-level objective with a per-agent KL penalty and a high-quantile monitor for $\KL^{\max}$ constraint:
\begin{equation}
\begin{split}
\mathcal{L}_i = &\; \E\!\left[\min\!\left\{r_i(\tau)\,\tilde A_{g},\; \exp\!\big(\mathrm{clip}(u_i(\tau),\log(1-\varepsilon),\log(1+\varepsilon))\big)\,\tilde A_{g}\right\}\right] \\
&\; -\;\beta\cdot \E_{s}\!\left[\KL\!\big(\pi^{\sigma(i)}(\cdot\mid s)\,\Vert\,\pi^{\sigma(i)}_{\mathrm{cur}}(\cdot\mid s)\big)\right],
\end{split}
\end{equation}
with $\varepsilon=0.2$ by default. The penalty coefficient $\beta$ is adapted online, and updates are backtracked whenever the empirical $(1-\alpha)$-quantile of per-state KL exceeds the target radius $\delta_i$. This enforces the per-agent trust region in the sense required by Sec.~\ref{sec:theory}: the quantile controller yields a high-probability relaxation of $\KL^{\max}$ that maps the theoretical radius to an effective $\delta_i+\epsilon(N,\eta)$, closing the gap between implementation and the bounds in Theorems~\ref{thm:single}--\ref{thm:joint}.

\begin{algorithm}[t]
\caption{Sequential Agent Tuning (SAT)}
\label{alg:sat}
\begin{algorithmic}[1]
\State \textbf{Input:} Team $\{\pi^{(j)}\}_{j=1}^{n}$, prompts $\mathcal{D}$, KL radii $\{\delta_j\}$, group size $G_{\text{grp}}$, clip $\varepsilon$, penalty $\beta$
\State \textbf{Initialize:} $\pi_{\text{cur}} \gets (\pi^{(1)}, \ldots, \pi^{(n)})$
\For{stage $k = 1, 2, \ldots$}
  \State $\mathcal{B} \gets \text{ROLLOUT}(\pi_{\text{cur}}, \mathcal{D})$; \quad $\sigma \gets \text{ORDERAGENTS}(\mathcal{B})$ {\color{red}\Comment{On-policy data; update order}}
  \For{$i = 1$ to $n$}
    \State $\hat{\pi}^{\,i-1} \gets (\pi^{\sigma(1)}_{\text{tar}}, \ldots, \pi^{\sigma(i-1)}_{\text{tar}}, \pi^{\sigma(i)}_{\text{cur}}, \ldots, \pi^{\sigma(n)}_{\text{cur}})$ 
    \State $\{\tilde{A}_g\} \gets \text{SEQADVANTAGES}(\mathcal{B}, \hat{\pi}^{\,i-1}, G_{\text{grp}})$ {\color{red}\Comment{GAE + group-norm + clip to $A_{\max}$}}
    \State $\mathcal{L}_i \gets \mathbb{E}[\min\{r_i \tilde{A}_g, \exp(\text{clip}(u_i, \log(1 \pm \varepsilon))) \tilde{A}_g\}] - \beta \mathbb{E}_s[D_{\text{KL}}(\pi^{\sigma(i)} \| \pi^{\sigma(i)}_{\text{cur}})]$
    \State $\pi^{\sigma(i)}_{\text{tar}} \gets \text{OPTIMIZE}(\mathcal{L}_i)$ {\color{red}\Comment{Trust-region update}}
    \If{$\text{Quantile}_{1-\alpha}[D_{\text{KL}}(\pi^{\sigma(i)}_{\text{tar}} \| \pi^{\sigma(i)}_{\text{cur}})] > \delta_i$} \textbf{backtrack} and increase $\beta$ {\color{red}\Comment{Enforce $D_{\text{KL}}^{\max} \le \delta_i$}}
    \EndIf
    \State $\pi^{\sigma(i)}_{\text{cur}} \gets \pi^{\sigma(i)}_{\text{tar}}$
  \EndFor
  \State $\pi_{\text{cur}} \gets (\pi^{(1)}_{\text{cur}}, \ldots, \pi^{(n)}_{\text{cur}})$ {\color{red}\Comment{Complete stage: $\bar{\pi}$ in Thm.~\ref{thm:joint}}}
\EndFor
\State \textbf{return} $\pi_{\text{cur}}$
\end{algorithmic}
\end{algorithm}

\paragraph{Plug-and-play agent upgrades.}
When replacing agent $j$ with a stronger pretrained model $\pi_{\text{pre}}^{(j)}$, we perform a Stage-0 alignment step that projects $\pi_{\text{pre}}^{(j)}$ onto the trust region around the $\pi_{\mathrm{cur}}^{(j)}$:
\begin{equation}
\begin{aligned}
\pi_{\text{new}}^{(j)}
={}&\arg\min_{\pi}\;\E_{s\sim d^{\pi_{\mathrm{cur}}}}\!\left[\KL\big(\pi(\cdot\mid s)\,\Vert\,\pi_{\text{pre}}^{(j)}(\cdot\mid s)\big)\right] \\
&\text{subject to}\quad
\KL\!\big(\pi(\cdot\mid s)\,\Vert\,\pi_{\mathrm{cur}}^{(j)}(\cdot\mid s)\big)\le \delta_0(s)\;\;\forall s.
\end{aligned}
\end{equation}
The Lagrangian yields a closed-form geometric mixture for state $s$:
\begin{equation}
\pi_{\text{new}}^{(j)}(a\mid s) = \frac{\pi_{\text{pre}}^{(j)}(a\mid s)^{1/(1+\lambda(s))}\,\big(\pi_{\mathrm{cur}}^{(j)}(a\mid s)\big)^{\lambda(s)/(1+\lambda(s))}}{\sum_{a'}\pi_{\text{pre}}^{(j)}(a'\mid s)^{1/(1+\lambda(s))}\,\big(\pi_{\mathrm{cur}}^{(j)}(a'\mid s)\big)^{\lambda(s)/(1+\lambda(s))}},
\end{equation}
where the state-dependent Lagrange multiplier $\lambda(s)\ge 0$ is chosen by binary search to satisfy the per-state KL constraint with equality when the pretrained policy would otherwise violate it. This initialization ensures $\pi_{\text{new}}^{(j)}$ starts within the trust region, so Theorems~\ref{thm:single}--\ref{thm:joint} remain valid. Moreover, if $\pi_{\text{pre}}^{(j)}$ is superior, the stage surrogate $L_i^{\textsc{seq}}$ improves at fixed $\delta_i$, or the same surrogate value can be achieved with tighter radii, reducing the occupancy-shift. 
\section{Experiments}
\label{sec:experiment}

We evaluate SAT across three domains: mathematical reasoning, active reasoning, and planning. Our experiments address three questions in a single unified protocol: whether SAT-trained small-model teams can match or exceed much larger monolithic models, whether empirical improvements align with the theory in Sec.\ref{sec:theory}, and how plug-and-play upgrades affect performance under fixed trust-region budgets. Unless explicitly noted, we follow the masked-activation product-policy view for action composition.

\subsection{Experimental Setup}
\label{sec:exp-setup}

\subsubsection{Models and Datasets}
We evaluate SAT by finetuning three small LLMs with parameters ranging from $1.5$B to $8$B: including LLaMA-3.3 (3B-instruct) \citep{dubey2024llama}, Qwen2.5 (1.5B-instruct, 3B-instruct, 7B-instruct) \citep{qwen2025qwen25technicalreport}, and Qwen3 (4B-instruct, 8B-instruct) \citep{yang2025qwen3technicalreport}. To demonstrate that SAT-trained small models can achieve competitive performance, we compare them against larger baseline models ranging from $30$B to $70$B parameters, such as LLaMA-3.3 (70B-instruct), Qwen2.5 (32B-instruct, 72B-instruct), and Qwen3 (30B-A3B-Instruct-2507, 32B), as well as publicly available thinking/non-thinking systems where applicable. For mathematical reasoning, we use AIME 2024/2025\citep{aime2024}, ZebraLogic \citep{lin2025zebralogic}, and MATH-500 \citep{hendrycks2021math}. For active reasoning, we utilize ARBench \citep{zhou2025passive} to evaluate the capacity of LLMs to formulate suitable questions to acquire additional information. For planning, we use AutoLogi\citep{zhu2025autologiautomatedgenerationlogic}and PlanBench \citep{valmeekam2024planbench}, including BlocksWorld and logistics domains. If the exact released parameter counts of some checkpoints differ slightly from the labels above, this variance does not alter the evaluation protocol; the citations remain as the authoritative references.

\subsubsection{Training Data.}
We prepare our training code with the VeRL framework \citep{sheng2024hybridflow} with temperature set to $0.8$, top\_p set to $1.0$, and maximum output length set to $32{,}768$ tokens for inference. Due to the high variance of the outputs from reasoning models, we report avg@K (pass@1 performance averaged over $K$ outputs) and pass@K for each benchmark. For benchmarks with few samples (AIME24/25 and ZebraLogic), we set a larger $K=64$. We use $K=25$ for ARBench, $K=8$ for planning benchmarks, and $K=4$ for MATH-500. To ensure accurate evaluation, we adopt the verification functions from DeepScaleR for mathematics problems. However, whether all external baselines used the same verifiers is uncertain; therefore, we retain their reported numbers and citations as is. For math reasoning ability, we use datasets from DeepScaleR~\citep{luo2025deepcoder} and DAPO \citep{yu2025dapo}; for active reasoning and planning, we use the training sets from ARBench and PlanBench directly.

\subsection{Main Results}
\label{sec:main-result}

Table~\ref{tab:overall} presents performance comparisons across general reasoning, active reasoning, and planning. A team of three Qwen3-4B models (12B total parameters) achieves $70.3\%$ on AIME24, $63.4\%$ on AIME25, and $86.1\%$ on MATH-500, while the three-agent LLaMA 3.1-8B SAT configuration (24B total) reaches $86.3\%$ (AIME24), $76.7\%$ (AIME25), $99.1\%$ (MATH-500), $93.1\%$ (ZebraLogic), and $42.7\%$ (PlanBench-BW NL). These results indicate that SAT-trained teams of small models can be competitive with, and in some cases surpass, much larger baselines in the listed settings, while using substantially fewer parameters. The precise relative ranking varies by benchmark, and we maintain every cited baseline as reported. The empirical trends are consistent with the theoretical guidance that per-agent KL radii and sequence-aware estimation control drift, yielding near-monotonic stage-wise improvements.
\begin{table*}[htbp]
\centering
\captionsetup{font=small} 
  \caption{Performance comparison across general reasoning, active reasoning, and planning. Publicly verifiable entries are kept as reported. Baseline results for Qwen3 and GPT-4 are from \citep{yang2025qwen3technicalreport}. We mark the best result in \textbf{bold} and the second-best with \underline{underline}.}
\label{tab:overall}
\small
\begin{tabular}{lccccccccccc}
\toprule
\multirow{2}{*}{Method} & \multirow{2}{*}{Size} & \multicolumn{5}{c}{General Reasoning (\%)} & \multicolumn{3}{c}{Active Reasoning ( \%)} & \multicolumn{1}{c}{Planning (\%)} \\
\cmidrule(lr){3-7}\cmidrule(lr){8-10}\cmidrule(lr){11-11}
& & AIME24 & AIME25 & MATH-500 & ZebraLogic & AutoLogi & DC & SP & GN & PlanBench-BW (NL) \\
\midrule
\multicolumn{11}{l}{\textit{Base Model Baselines}} \\
Qwen3-Base & 14B & 31.7 & 23.3 & 90.1 & 33.1 & 79.1 & 37.9 & 32.7 & 36.5 & 29.7 \\
Qwen3-Base & 32B & 31.0 & 20.2 & 88.6 & 29.2 & 78.5 & 41.3 & 35.7 & 39.1 & 34.3 \\
Qwen2.5-Base & 72B & 18.9 & 15.0 & 83.6 & 26.6 & 76.7 & 39.5 & 34.1 & 38.1 & 33.7 \\
GPT-4o-mini-2024-07-18 & / & 8.1 & 8.8 & 78.2 & 20.1 & 62.5 & 44.0 & 40.8 & 43.6 & 34.7 \\
\midrule
\multicolumn{11}{l}{\textit{Thinking Model Baselines}} \\
Qwen3-A3B & 30B & 80.4 & 70.9 & \underline{98.0} & \underline{89.5} & 88.1 & 43.1 & 38.4 & 42.3 & 39.1 \\
QwQ & 32B & 79.5 & 69.5 & 98.0 & 76.8 & 86.3 & 41.9 & 36.1 & 39.5 & 37.9 \\
Qwen3 (thinking) & 32B & \underline{81.4} & 72.9 & 97.2 & 88.8 & 87.3 & 42.7 & 38.9 & 41.1 & \underline{40.3} \\
DeepSeek-R1 Distill Llama & 70B & 70.0 & 56.3 & 94.5 & 71.3 & 83.5 & 42.1 & 35.4 & 40.7 & 39.1 \\
OpenAI o3-mini (medium) & / & 79.6 & \underline{74.8} & \underline{98.0} & 88.9 & 86.3 & \textbf{46.7} & \textbf{43.1} & \textbf{45.1} & 41.2 \\
\midrule
\multicolumn{11}{l}{\textit{Single-Agent Fine-tuning Baselines}} \\
Qwen3-Base-GRPO & 8B & 39.1 & 27.9 & 90.7 & 35.4 & 80.3 & 40.1 & 35.2 & 38.9 & 33.1 \\
Qwen3-Base-DAPO & 8B & 41.3 & 30.1 & 91.5 & 36.7 & 80.9 & 40.7 & 35.9 & 39.3 & 33.5 \\
\midrule
\multicolumn{11}{l}{\textit{Multi-Agent with Fine-tuning Baselines. A Judge serves as a coordinator.}} \\
Qwen3-Base Debate & 3$\times$8B & 68.7 & 59.1 & 95.2 & 84.1 & 85.1 & 43.1 & 38.0 & 42.3 & 38.3 \\
Qwen3-Base Debate + Judge & 3$\times$8B & 71.1 & 61.7 & 95.9 & 86.3 & 85.9 & 44.3 & 39.4 & 43.5 & 39.1 \\
Qwen3-Base Role-play & 3$\times$8B & 69.9 & 60.5 & 95.7 & 85.1 & 85.5 & 43.7 & 38.7 & 43.1 & 38.7 \\
\midrule
\multicolumn{11}{l}{\textit{SAT (Ours)}} \\
Llama 3.1-Base SAT & 3$\times$8B & 70.3 & 63.4 & 86.1 & 80.7 & 83.9 & 35.9 & 32.7 & 34.9 & 37.1 \\
Qwen3-Base SAT & 3$\times$4B & \textbf{86.3} & \textbf{76.7} & \textbf{99.1} & \textbf{93.1} & \textbf{89.1} & \underline{45.3} & \underline{41.5} & \underline{43.9} & \textbf{42.7} \\
\bottomrule
\end{tabular}
\end{table*}

\subsection{Heterogeneous Models Analysis}
\label{sec:hetero-analysis}
Table~\ref{tab:heterogeneous} explores parameter-size and model-family heterogeneity under SAT without predefined roles. Performance improves from $22.5$ (2$\times$0.6B+1.7B configuration) to $52.1$ (2$\times$4B+8B configuration) as total capacity increases, with clear diminishing returns: moving from sub-2B to 4B models yields substantial improvements, while adding 8B models provides more modest gains. This pattern aligns with our theoretical analysis, which shows super-linear prefix penalties for large cumulative updates. Model selection significantly impacts team performance—the heterogeneous configuration (2$\times$Qwen3-4B + LLaMA-8B) achieves an average score of $50.4$ relative to the homogeneous 3$\times$Qwen3-4B baseline ($50.7$). Adding a stronger agent (Gemma-3-12B-IT) improves performance to $53.4$, although at an increased computational cost. Mixing thinking and non-thinking agents often degrades performance, consistent with the need for agents trained under compatible sequence-level surrogates.

\begin{table*}[h!]
\centering
\captionsetup{font=small} 
\caption{Heterogeneous team configurations under SAT (3 agents, no predefined roles). 
Datasets: AIME24, AIME25, ARBench-DC (process@25), PlanBench-BW (NL). 
\textit{Params (B)} is the sum of the three agents' parameters (no weight sharing). 
\textit{Heterogeneity tags}: size = different parameter scales (capacity), 
model = different model families / pretraining methods and corpora, 
thinking = different thinking styles (e.g., deliberate/CoT-distilled vs non-thinking). 
We mark the best column for the average result in \textbf{bold} and the second-best with \underline{underline}.}
\label{tab:heterogeneous}
\small
\begin{tabular}{lcccccc}
\toprule
Configuration & Params (B) & AIME24 & AIME25 & ARBench-DC & PlanBench-BW (NL) & Average \\
\midrule
\multicolumn{7}{l}{\textit{Parameter-Size Heterogeneity}} \\
2$\times$Qwen3-0.6B Base + Qwen3-1.7B  Base            & 3.0  & 22.1 & 17.3 & 26.1 & 24.5 & 22.5 \\
Qwen3-0.6B Base+ Qwen3-1.7B Base+ Qwen3-4B  Base          & 6.3  & 49.7 & 38.5 & 31.5 & 32.9 & 38.2 \\
2$\times$Qwen3-1.7B Base+ Qwen3-4B Base               & 7.4  & 55.9 & 44.3 & 33.1 & 34.9 & 42.1 \\
Qwen3-1.7B Base+ Qwen3-4B Base+ Qwen3-8B Base              & 13.7 & 66.9 & 56.3 & 35.5 & 36.7 & \underline{48.9} \\
2$\times$Qwen3-4B Base+ Qwen3-8B Base                 & 16.0 & 72.1 & 61.1 & 36.7 & 38.5 & \textbf{52.1} \\
\midrule
\multicolumn{7}{l}{\textit{Model Heterogeneity}} \\
2$\times$Qwen3-4B Base+ Qwen2.5-3B Base               & 11.0 & 68.1 & 57.3 & 35.1 & 36.3 & 49.2 \\
LLaMA~3.1-8B Base+ LLaMA~3.3-4B Base+ Qwen2.5-3B Base     & 15.0 & 64.7 & 54.3 & 34.7 & 35.9 & 47.4 \\
2$\times$Qwen3-4B Base+ LLaMA~3.1-8B  Base            & 16.0 & 69.5 & 58.7 & 35.7 & 37.7 & \underline{50.4} \\
Qwen3-4B Base+ LLaMA~3.1-8B Base+ Gemma-3-12B-IT      & 24.0 & 74.3 & 64.1 & 36.9 & 38.3 & \textbf{53.4} \\
\midrule
\multicolumn{7}{l}{\textit{Thinking Modes + Non-thinking Mix}} \\
DeepSeek-R1-Distill-(Qwen-1.5B + Llama3.1-8B Base)$\downarrow$\\\quad + Qwen3-4B (non-thinking)& 13.5 & 58.3 & 47.1 & 33.3 & 35.1 & 43.5 \\
2$\times$Qwen3-0.6B (thinking) + Qwen3-4B (non-thinking)                                & 5.2  & 41.7 & 32.5 & 30.1 & 31.9 & 34.1 \\
2$\times$Qwen3-4B (thinking) + Qwen3-1.7B (non-thinking)                                & 9.7  & 62.9 & 52.3 & 34.3 & 36.1 & \underline{46.4} \\
2$\times$Qwen3-4B (thinking) + Qwen3-8B (non-thinking)                                  & 16.0 & 64.3 & 53.7 & 34.7 & 36.7 & \textbf{47.3} \\
\bottomrule
\end{tabular}
\end{table*}

\subsection{Stage-wise Improvement Analysis}

\begin{table*}[h!]
\centering
\captionsetup{font=small} 
\caption{Plug-and-play under SAT. The baseline is Qwen3 trained by SAT. ``Composite'' is the mean of {AIME24, AIME25, ARBench-DC} under the same evaluation protocol as Table~\ref{tab:overall}. We compare static heterogeneity (train from scratch) with PnP upgrades on the same baseline; KL is per-agent per-state; Viol.\% is the first-stage monotonicity violation rate; Cost is relative to baseline parameters. The best performance is in \textbf{bold}.}
\label{tab:plugin}

\small
\begin{tabular}{lcccccccccc}
\toprule
\multirow{2}{*}{Block / Strategy} & \multirow{2}{*}{Change} & \multirow{2}{*}{KL} & \multicolumn{3}{c}{Composite} & \multicolumn{3}{c}{$\Delta$ vs Baseline (AIME24/AIME25/ARBench-DC)} & \multirow{2}{*}{Viol.\%} & \multirow{2}{*}{Cost} \\
\cmidrule(lr){4-6}\cmidrule(lr){7-9}
 &  &  & Before & After & Gain & AIME24 & AIME25 & ARBench-DC &  &  \\
\midrule
\textit{Baseline}: Qwen3 SAT (3$\times$4B) & -- & -- & 56.5 & 56.5 & --   & --   & --   & --   & 2.7 & 1$\times$ \\
\midrule
\multicolumn{11}{l}{\textit{Static hetero (train from scratch)}} \\
Direct hetero (2$\times$4B + LLaMA~8B)            & +1$\times$ capacity & --    & 56.5 & 55.5 & -1.0 & -1.2 & -1.0 & -0.6 & --  & 1.33$\times$ \\
Direct hetero (2$\times$4B + Qwen3~8B)            & +1$\times$ capacity & --    & 56.5 & 57.5 & +1.0 & +1.4 & +1.2 & +0.4 & --  & 1.33$\times$ \\
Direct hetero (1$\times$4B + 2$\times$Qwen3~8B)   & +2$\times$ capacity & --    & 56.5 & 66.5 & +10.0& +11.8& +12.8& +5.6 & --  & 1.67$\times$ \\
\midrule
\multicolumn{11}{l}{\textit{PnP hetero (plug-and-play; continue SAT)}} \\
Replace any one with LLaMA~3.1-8B                 & +1$\times$ capacity & 0.010 & 56.5 & 55.9 & -0.6 & -0.8 & -0.6 & -0.2 & 4.7 & 1.33$\times$ \\
Replace any one with Qwen2.5-3B                   & -1$\times$ capacity & 0.010 & 56.5 & 55.5 & -1.0 & -1.2 & -1.2 & -0.6 & 3.9 & 0.92$\times$ \\
Replace any one with Qwen3-1.7B                   & -1$\times$ capacity & 0.010 & 56.5 & 54.8 & -1.7 & -2.2 & -2.2 & -0.6 & 3.7 & 0.81$\times$ \\
Replace any one with Qwen3-0.6B                   & -1$\times$ capacity & 0.010 & 56.5 & 51.9 & -4.6 & -5.8 & -6.0 & -2.0 & \underline{3.3} & 0.72$\times$ \\
Replace any one with Gemma-3-12B-IT               & +1$\times$ capacity & 0.010 & 56.5 & 59.7 & +3.2 & +4.0 & +4.8 & +1.0 & 4.1 & 1.67$\times$ \\
Replace two with Qwen3-8B                         & +2$\times$ capacity & 0.010 & 56.5 & \underline{66.9} & \underline{+10.4}& \underline{+12.0}& \underline{+13.2}& \underline{+6.0 }& \underline{3.3} & 1.67$\times$ \\
Full PnP to Qwen3 (3$\times$8B)                   & +3$\times$ capacity & 0.004 & 56.5 & \textbf{70.4} & \textbf{+13.9}& \textbf{+16.2}& \textbf{+17.6}& \textbf{+8.0} & \textbf{2.5} & 2.00$\times$ \\
\bottomrule
\end{tabular}
\end{table*}
Figure~\ref{fig:stage-curves} illustrates training dynamics and trust-region behavior across the full optimization horizon. Panels~(a) and~(b) show SAT training curves on AIME24 and ARBench, respectively, compared against DAPO and GRPO baselines using Qwen2.5-32B. Horizontal dashed lines denote the performance of the 3-round $3\times$GPT-4o-mini debate and GPT-o1-preview as reference points. On AIME24, SAT exhibits rapid early gains before stabilizing around 53\% accuracy, while ARBench demonstrates more gradual but steady improvement to 41\% success rate. Panel~(c) provides empirical validation of our trust-region analysis during AIME24 training. The violation rate, measuring the fraction of states where the per-state KL constraint is exceeded, scales approximately as $\sqrt{\delta}$ (black dashed line), consistent with the occupancy-shift penalty in Theorem~\ref{thm:joint}. Across multiple agents and training stages, violations concentrate in early steps and diminish as agents align, matching Theorem~\ref{thm:finite-sample}.
%

\begin{figure*}[h!]
\centering
\captionsetup{font=small}
\includegraphics[width=\textwidth]{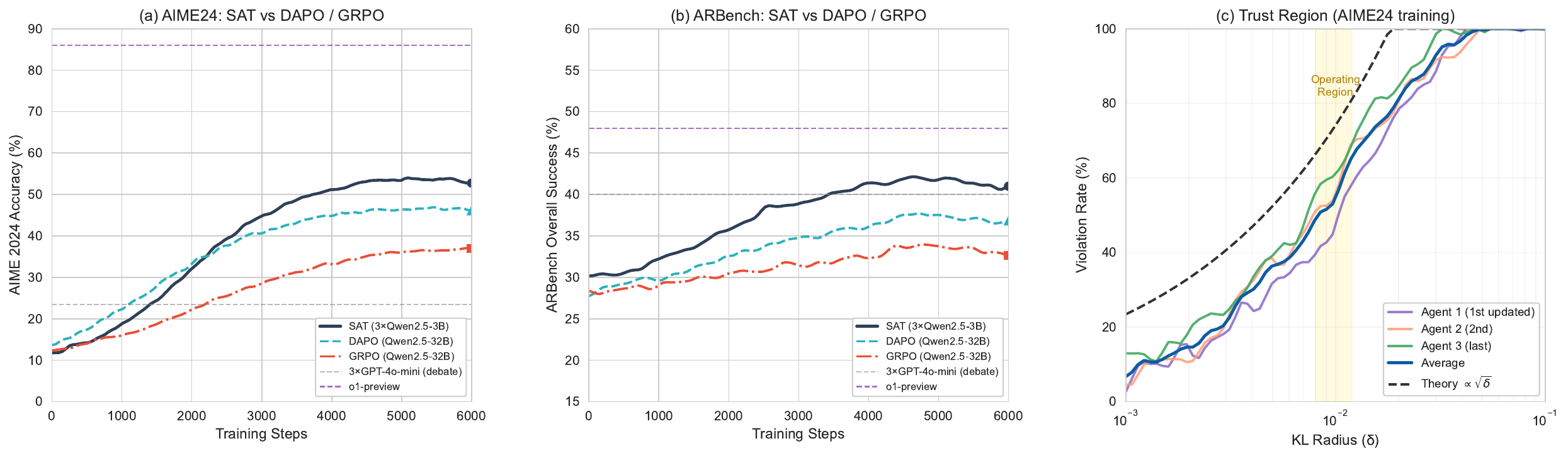}
\caption{\textbf{Stage-wise performance and trust-region validation with SAT.} 
(a)~AIME24 2024 accuracy over training steps; (b)~ARBench overall success rate; (c)~trust-region violation rate vs.\ KL radius~$\delta$ on AIME24 training. 
Horizontal dashed lines in (a) and (b) denote 3-round $3\times$GPT-4o-mini debate and GPT-o1-preview baselines. SAT is compared against Qwen2.5-32B trained with DAPO and GRPO. SAT demonstrates rapid early gains on AIME24 and steady improvements on ARBench. Panel~( ~(c) shows that violation rate scales as $\sqrt{\delta}$ (dashed theory line), with the operating region (yellow) achieving low violations while enabling effective exploration, consistent with Theorems~\ref{thm:joint} and~\ref{thm:finite-sample}.
}
\label{fig:stage-curves}
\end{figure*}
\subsection{Plug-and-Play Upgrades}
\label{sec:plugin}

Table~\ref{tab:plugin} evaluates plug-and-play capabilities under SAT without role assignments, comparing PnP upgrades against static heterogeneous training from scratch using the composite metric. Single-agent replacement with LLaMA~3.1-8B slightly decreases the composite from $55.2$ to $54.6$ under $\delta=0.010$, while replacing two agents with 8B models increases it to $65.6$. Full PnP to three 8B models reaches $69.1$ composite with the lowest violation rate ($2.5\%$) at $2.00\times$ cost. Downgrade experiments quantify capacity performance trade offs: replacing with Qwen3-0.6B reduces composite to $50.6$ at $0.72\times$ cost. 

\section{Conclusion}
We have introduced Sequential Agent Tuning (SAT), a coordinator-free framework that empowers teams of small LLMs to surpass the performance of larger models. Our method integrates sequential agent updates with a sequence-aware advantage estimator and per-agent KL trust regions, ensuring stable and monotonic improvements. The central element of our theoretical contribution is the principle of “plug-and-play” invariance, which allows individual agents to be upgraded without incurring the cost of retraining the entire team. We further validated our plug-and-play theory by demonstrating that modular agent upgrades yield significant performance gains, in line with analytical expectations. These results position SAT as a practical and scalable path for deploying high-performance AI systems, especially in resource-constrained environments. We hope this work encourages a shift from simply scaling monolithic models to strategically advancing teams of smaller agents for LLM community.

\begin{acks}
Y. Fan's work is not related to the position at Amazon Inc. Y. Xu's work is funded by the National Science Foundation (NSF) under grant NSF IIS1910794 and DMS-2406896;
\end{acks}

\clearpage
\bibliographystyle{ACM-Reference-Format} 
\bibliography{sample}

@misc{yang2025qwen3technicalreport,
      title={Qwen3 Technical Report}, 
      author={An Yang and Anfeng Li and Baosong Yang and Beichen Zhang and Binyuan Hui and Bo Zheng and Bowen Yu and Chang Gao and Chengen Huang and Chenxu Lv and Chujie Zheng and Dayiheng Liu and Fan Zhou and Fei Huang and Feng Hu and Hao Ge and Haoran Wei and Huan Lin and Jialong Tang and Jian Yang and Jianhong Tu and Jianwei Zhang and Jianxin Yang and Jiaxi Yang and Jing Zhou and Jingren Zhou and Junyang Lin and Kai Dang and Keqin Bao and Kexin Yang and Le Yu and Lianghao Deng and Mei Li and Mingfeng Xue and Mingze Li and Pei Zhang and Peng Wang and Qin Zhu and Rui Men and Ruize Gao and Shixuan Liu and Shuang Luo and Tianhao Li and Tianyi Tang and Wenbiao Yin and Xingzhang Ren and Xinyu Wang and Xinyu Zhang and Xuancheng Ren and Yang Fan and Yang Su and Yichang Zhang and Yinger Zhang and Yu Wan and Yuqiong Liu and Zekun Wang and Zeyu Cui and Zhenru Zhang and Zhipeng Zhou and Zihan Qiu},
      year={2025},
      eprint={2505.09388},
      archivePrefix={arXiv},
      primaryClass={cs.CL},
      url={https://arxiv.org/abs/2505.09388}, 
}

@article{luo2025deepcoder,
  title={Deepcoder: A fully open-source 14b coder at o3-mini level},
  author={Luo, Michael and Tan, Sijun and Huang, Roy and Patel, Ameen and Ariyak, Alpay and Wu, Qingyang and Shi, Xiaoxiang and Xin, Rachel and Cai, Colin and Weber, Maurice and others},
  journal={Notion Blog},
  year={2025}
}

@misc{zhu2025autologiautomatedgenerationlogic,
      title={AutoLogi: Automated Generation of Logic Puzzles for Evaluating Reasoning Abilities of Large Language Models}, 
      author={Qin Zhu and Fei Huang and Runyu Peng and Keming Lu and Bowen Yu and Qinyuan Cheng and Xipeng Qiu and Xuanjing Huang and Junyang Lin},
      year={2025},
      eprint={2502.16906},
      archivePrefix={arXiv},
      primaryClass={cs.CL},
      url={https://arxiv.org/abs/2502.16906}, 
}

@inproceedings{
lin2025zebralogic,
title={ZebraLogic: On the Scaling Limits of {LLM}s for Logical Reasoning},
author={Bill Yuchen Lin and Ronan Le Bras and Kyle Richardson and Ashish Sabharwal and Radha Poovendran and Peter Clark and Yejin Choi},
booktitle={Forty-second International Conference on Machine Learning},
year={2025},
url={https://openreview.net/forum?id=sTAJ9QyA6l}
}

@inproceedings{
yi2025from,
title={From Debate to Equilibrium: Belief\nobreakdash-Driven Multi\nobreakdash-Agent {LLM} Reasoning via Bayesian Nash Equilibrium},
author={Xie Yi and Zhanke Zhou and Chentao Cao and Qiyu Niu and Tongliang Liu and Bo Han},
booktitle={Forty-second International Conference on Machine Learning},
year={2025},
url={https://openreview.net/forum?id=RQwexjUCxm}
}

@article{sheng2024hybridflow,
  title   = {HybridFlow: A Flexible and Efficient RLHF Framework},
  author  = {Guangming Sheng and Chi Zhang and Zilingfeng Ye and Xibin Wu and Wang Zhang and Ru Zhang and Yanghua Peng and Haibin Lin and Chuan Wu},
  year    = {2024},
  journal = {arXiv preprint arXiv: 2409.19256}
}

@article{yu2025dapo,
  title={Dapo: An open-source llm reinforcement learning system at scale},
  author={Yu, Qiying and Zhang, Zheng and Zhu, Ruofei and Yuan, Yufeng and Zuo, Xiaochen and Yue, Yu and Dai, Weinan and Fan, Tiantian and Liu, Gaohong and Liu, Lingjun and others},
  journal={arXiv preprint arXiv:2503.14476},
  year={2025}
}

@article{zhou2025passive,
  title={From Passive to Active Reasoning: Can Large Language Models Ask the Right Questions under Incomplete Information?},
  author={Zhou, Zhanke and Feng, Xiao and Zhu, Zhaocheng and Yao, Jiangchao and Koyejo, Sanmi and Han, Bo},
  journal={arXiv preprint arXiv:2506.08295},
  year={2025}
}

@article{dubey2024llama,
  title={The llama 3 herd of models},
  author={Dubey, Abhimanyu and Jauhri, Abhinav and Pandey, Abhinav and Kadian, Abhishek and Al-Dahle, Ahmad and Letman, Aiesha and Mathur, Akhil and Schelten, Alan and Yang, Amy and Fan, Angela and others},
  journal={arXiv e-prints},
  pages={arXiv--2407},
  year={2024}
}

@inproceedings{leviathan2023fast,
  title={Fast inference from transformers via speculative decoding},
  author={Leviathan, Yaniv and Kalman, Matan and Matias, Yossi},
  booktitle={International Conference on Machine Learning},
  pages={19274--19286},
  year={2023},
  organization={PMLR}
}

@article{munos2016safe,
  title={Safe and efficient off-policy reinforcement learning},
  author={Munos, R{\'e}mi and Stepleton, Tom and Harutyunyan, Anna and Bellemare, Marc},
  journal={Advances in neural information processing systems},
  volume={29},
  year={2016}
}

@misc{qwen2025qwen25technicalreport,
      title={Qwen2.5 Technical Report}, 
      author={Qwen and : and An Yang and Baosong Yang and Beichen Zhang and Binyuan Hui and Bo Zheng and Bowen Yu and Chengyuan Li and Dayiheng Liu and Fei Huang and Haoran Wei and Huan Lin and Jian Yang and Jianhong Tu and Jianwei Zhang and Jianxin Yang and Jiaxi Yang and Jingren Zhou and Junyang Lin and Kai Dang and Keming Lu and Keqin Bao and Kexin Yang and Le Yu and Mei Li and Mingfeng Xue and Pei Zhang and Qin Zhu and Rui Men and Runji Lin and Tianhao Li and Tianyi Tang and Tingyu Xia and Xingzhang Ren and Xuancheng Ren and Yang Fan and Yang Su and Yichang Zhang and Yu Wan and Yuqiong Liu and Zeyu Cui and Zhenru Zhang and Zihan Qiu},
      year={2025},
      eprint={2412.15115},
      archivePrefix={arXiv},
      primaryClass={cs.CL},
      url={https://arxiv.org/abs/2412.15115}, 
}

@article{openai2023gpt4,
  title        = {{GPT-4} Technical Report},
  author       = {{OpenAI}},
  journal      = {arXiv preprint arXiv:2303.08774},
  year         = {2023},
  url          = {https://arxiv.org/abs/2303.08774}
}

@article{patterson2021carbon,
  title   = {Carbon Emissions and Large Neural Network Training},
  author  = {Patterson, David and Gonzalez, Joseph and Le, Quoc and Liang, Chen and Munguia, Lluis-Miquel and Rothchild, Daniel and So, David and Texier, Maud and Dean, Jeff},
  journal = {arXiv preprint arXiv:2104.10350},
  year    = {2021}
}

@article{alizadeh2024flash,
  title   = {{LLM} in a Flash: Efficient Large Language Model Inference with Limited Memory},
  author  = {Alizadeh, Keivan and Mirzadeh, Iman and Belenko, Dmitry and Khatamifard, S. Karen and Cho, Minsik and Del Mundo, Carlo C. and Rastegari, Mohammad and Farajtabar, Mehrdad},
  journal = {Proceedings of the 62nd Annual Meeting of the ACL},
  year    = {2024},
  url     = {https://aclanthology.org/2024.acl-long.678.pdf}
}

@article{wei2022chain,
  title   = {Chain-of-Thought Prompting Elicits Reasoning in Large Language Models},
  author  = {Wei, Jason and Wang, Xuezhi and Schuurmans, Dale and others},
  journal = {arXiv preprint arXiv:2201.11903},
  year    = {2022}
}

@inproceedings{wang2023self,
  title     = {Self-Consistency Improves Chain of Thought Reasoning in Language Models},
  author    = {Wang, Xuezhi and Wei, Jason and Schuurmans, Dale and Le, Quoc and Chi, Ed and Narang, Sharan and Chowdhery, Aakanksha and Zhou, Denny},
  booktitle = {International Conference on Learning Representations (ICLR)},
  year      = {2023},
  url       = {https://arxiv.org/abs/2203.11171}
}

@article{yao2023treeofthoughts,
  title   = {Tree of Thoughts: Deliberate Problem Solving with Large Language Models},
  author  = {Yao, Shunyu and Yu, Dian and Zhao, Jeffrey and Shafran, Izhak and Griffiths, Thomas L. and Cao, Yuan and Narasimhan, Karthik},
  journal = {arXiv preprint arXiv:2305.10601},
  year    = {2023},
  url     = {https://arxiv.org/abs/2305.10601}
}

@inproceedings{shinn2023reflexion,
  title     = {Reflexion: Language Agents with Verbal Reinforcement Learning},
  author    = {Shinn, Noah and Cassano, Federico and Gopinath, Ashwin and Berman, Edward and Narasimhan, Karthik and Yao, Shunyu},
  booktitle = {Advances in Neural Information Processing Systems (NeurIPS)},
  year      = {2023},
  url       = {https://arxiv.org/abs/2303.11366}
}

@article{wu2023autogen,
  title   = {AutoGen: Enabling Next-Gen {LLM} Applications via Multi-Agent Conversation},
  author  = {Wu, Qingyun and Bansal, Gagan and Zhang, Jieyu and Wu, Yiran and Li, Beibin and Zhu, Erkang and Jiang, Li and Zhang, Xiaoyun and Zhang, Shaokun and Liu, Jiale and Awadallah, Ahmed Hassan and White, Ryen W. and Burger, Doug and Wang, Chi},
  journal = {arXiv preprint arXiv:2308.08155},
  year    = {2023}
}

@article{li2023camel,
  title   = {{CAMEL}: Communicative Agents for "Mind" Exploration of Large-Scale Language Model Society},
  author  = {Li, Guohao and Hammoud, Hasan Abed Al Kader and Itani, Hani and Khizbullin, Dmitrii and Ghanem, Bernard},
  journal = {arXiv preprint arXiv:2303.17760},
  year    = {2023}
}

@article{wang2024moa,
  title   = {Mixture-of-Agents Enhances Large Language Model Capabilities},
  author  = {Wang, Junlin and Wang, Jue and Athiwaratkun, Ben and Zhang, Ce and Zou, James},
  journal = {arXiv preprint arXiv:2406.04692},
  year    = {2024},
  url     = {https://arxiv.org/abs/2406.04692}
}

@inproceedings{yao2022react,
  title       = {ReAct: Synergizing Reasoning and Acting in Language Models},
  author      = {Yao, Shunyu and Zhao, Jeffrey and Yu, Dian and Du, Nan and Shafran, Izhak and Narasimhan, Karthik and Cao, Yuan},
  booktitle   = {International Conference on Learning Representations (ICLR)},
  year        = {2023},
  eprint      = {2210.03629},
  archivePrefix = {arXiv},
  url         = {https://arxiv.org/abs/2210.03629}
}

@inproceedings{kakade2002cpi,
  title     = {Approximately Optimal Approximate Reinforcement Learning},
  author    = {Kakade, Sham and Langford, John},
  booktitle = {Proceedings of the 19th International Conference on Machine Learning (ICML)},
  year      = {2002},
  note      = {Often cited via Conservative Policy Iteration (CPI)}
}

@inproceedings{schulman2015trpo,
  title     = {Trust Region Policy Optimization},
  author    = {Schulman, John and Levine, Sergey and Abbeel, Pieter and Jordan, Michael I. and Moritz, Philipp},
  booktitle = {Proceedings of the 32nd International Conference on Machine Learning (ICML)},
  year      = {2015},
  url       = {https://proceedings.mlr.press/v37/schulman15.html}
}

@inproceedings{schulman2016gae,
  title        = {High-Dimensional Continuous Control Using Generalized Advantage Estimation},
  author       = {Schulman, John and Moritz, Philipp and Levine, Sergey and Jordan, Michael I. and Abbeel, Pieter},
  booktitle    = {International Conference on Learning Representations (ICLR)},
  year         = {2016},
  url          = {https://arxiv.org/abs/1506.02438}
}

@article{shao2024deepseekmath,
  title   = {DeepSeekMath: Pushing the Limits of Mathematical Reasoning in Open Language Models},
  author  = {Shao, Zhenda and others},
  journal = {arXiv preprint arXiv:2402.03300},
  year    = {2024},
  note    = {Introduces Group-Relative Policy Optimization (GRPO)}
}

@misc{aime2024,
  title        = {2024 AIME I: Problems and Solutions},
  author       = {{Art of Problem Solving}},
  year         = {2024},
  howpublished = {\url{https://artofproblemsolving.com/wiki/index.php/2024_AIME_I}},
  note         = {Last accessed: 2025-09-23}
}

@inproceedings{hendrycks2021math,
  title     = {Measuring Mathematical Problem Solving With the MATH Dataset},
  author    = {Hendrycks, Dan and Burns, Collin and Kadavath, Saurav and Arora, Akul and Basart, Steven and Tang, Eric and Song, Dawn and Steinhardt, Jacob},
  booktitle = {NeurIPS 2021 Datasets and Benchmarks Track},
  year      = {2021},
  url       = {https://arxiv.org/abs/2103.03874}
}

@inproceedings{xie2025acorn,
    title={ACORN: Acyclic Coordination with Reachability Network to Reduce Communication Redundancy in Multi-Agent Systems},
    author={Yi Xie and Ziqing Zhou and Chun Ouyang and Siao Liu and Linqiang Hu and Zhongxue Gan},
    booktitle={Proceedings of AAMAS},
    pages={2190--2198},
    year={2025}
}

@inproceedings{xie2025haer,
    title={Heuristics-Assisted Experience Replay Strategy for Cooperative Multi-Agent Reinforcement Learning},
    author={Yi Xie and Ziqing Zhou and Chun Ouyang and Siao Liu and Linqiang Hu and Zhongxue Gan},
    booktitle={Proceedings of AAMAS},
    pages={2798--2800},
    year={2025}
}

@inproceedings{liu2023improving,
    title={Improving Generalization in Visual Reinforcement Learning via Conflict-aware Gradient Agreement Augmentation},
    author={Siao Liu and Zhaoyu Chen and Yang Liu and Yuzheng Wang and Dingkang Yang and Zhile Zhao and Ziqing Zhou and Yi Xie and Wei Li and Wenqiang Zhang and Zhongxue Gan},
    booktitle={Proceedings of ICCV},
    pages={23436--23446},
    year={2023}
}

@article{liu2025grasp,
    title={Improving Robotic Grasp Detection Under Sparse Annotations Via Grasp Transformer With Pixel-Wise Contrastive Learning},
    author={Siao Liu and Yang Liu and Zhaoyu Chen and Ziqing Zhou and Zhile Zhao and Yi Xie and Wei Li and Zhongxue Gan},
    journal={IEEE Transactions on Industrial Electronics},
    year={2025},
    doi={10.1109/TIE.2025.3569940}
}

@article{xie2018block,
  title={A block coordinate ascent algorithm for mean-variance optimization},
  author={Xie, Tengyang and Liu, Bo and Xu, Yangyang and Ghavamzadeh, Mohammad and Chow, Yinlam and Lyu, Daoming and Yoon, Daesub},
  journal={Advances in Neural Information Processing Systems},
  volume={31},
  year={2018}
}

@inproceedings{zhang2021mean,
  title={Mean-variance policy iteration for risk-averse reinforcement learning},
  author={Zhang, Shangtong and Liu, Bo and Whiteson, Shimon},
  booktitle={Proceedings of the AAAI Conference on Artificial Intelligence},
  volume={35},
  number={12},
  pages={10905--10913},
  year={2021}
}

@article{valmeekam2024planbench,
  title   = {PlanBench: An Extensible Benchmark for Evaluating Large Language Models on Planning and Reasoning about Change},
  author  = {Valmeekam, Karthik and Marquez, Matthew and Olmo, Alberto and Sreedharan, Sarath and Kambhampati, Subbarao},
  journal = {arXiv preprint arXiv:2206.10498},
  year    = {2022},
  note    = {NeurIPS 2023 Poster; widely used in 2024--2025 planning evaluations},
  url     = {https://arxiv.org/abs/2206.10498}
}

@article{feng2023alphazero,
  title={Alphazero-like tree-search can guide large language model decoding and training},
  author={Feng, Xidong and Wan, Ziyu and Wen, Muning and McAleer, Stephen Marcus and Wen, Ying and Zhang, Weinan and Wang, Jun},
  journal={arXiv preprint arXiv:2309.17179},
  year={2023}
}

@article{schulman2017proximal,
  title={Proximal policy optimization algorithms},
  author={Schulman, John and Wolski, Filip and Dhariwal, Prafulla and Radford, Alec and Klimov, Oleg},
  journal={arXiv preprint arXiv:1707.06347},
  year={2017}
}

@inproceedings{espeholt2018impala,
  title={Impala: Scalable distributed deep-rl with importance weighted actor-learner architectures},
  author={Espeholt, Lasse and Soyer, Hubert and Munos, Remi and Simonyan, Karen and Mnih, Vlad and Ward, Tom and Doron, Yotam and Firoiu, Vlad and Harley, Tim and Dunning, Iain and others},
  booktitle={International conference on machine learning},
  pages={1407--1416},
  year={2018},
  organization={PMLR}
}

\end{document}